\documentclass{article}
\usepackage{spconf,amsmath,graphicx,stfloats,subfigure}
\usepackage{amsthm,bm,mathrsfs,amsfonts}
\usepackage{algorithmicx,algorithm}
\usepackage{booktabs}

\usepackage{color}

\newtheorem{theorem}{Theorem}
\newtheorem{lemma}{Lemma}
\newtheorem{defn}{Definition}
\usepackage{diagbox}
\usepackage{multirow}

\title{ADVERSARIAL DEFENSE VIA LOCAL FLATNESS REGULARIZATION}
\name{Jia Xu$^{1, \star}$\thanks{$\star$: equal contribution.} \qquad Yiming Li$^{1, \star}$ \qquad Yong Jiang$^{1, 2}$\qquad Shu-Tao Xia$^{1, 2}$\thanks{This work is supported by the National Science Foundation of China under Grant 6177127, the Natural Science Foundation of Zhejiang Province (LSY19A010002), the project ``PCL Future Greater-Bay Area Network Facilities for Large-scale Experiments and Applications (LZC0019)", and the R\&D Program of Shenzhen (JCYJ20180508152204044). Corresponding to Yong Jiang.}}
\address{$^{1}$Tsinghua Shenzhen International Graduate School,  Tsinghua University, China\\
$^{2}$PCL Research Center of Networks and Communications, Peng Cheng Laboratory, China\\
 \{xujia19, li-ym18\}@mails.tsinghua.edu.cn;
 \{jiangy, xiast\}@sz.tsinghua.edu.cn}
\begin{document}
%
\maketitle
\begin{abstract}
Adversarial defense is a popular and important research area. Due to its intrinsic mechanism, one of the most straightforward and effective ways of defending attacks is to analyze the property of loss surface in the input space. In this paper, we define the local flatness of the loss surface as the maximum value of the chosen norm of the gradient regarding to the input within a neighborhood centered on the benign sample, and discuss the relationship between the local flatness and adversarial vulnerability. Based on the analysis, we propose a novel defense approach via regularizing the local flatness, dubbed local flatness regularization (LFR). We also demonstrate the effectiveness of the proposed method from other perspectives, such as human visual mechanism, and analyze the relationship between LFR and other related methods theoretically. Experiments are conducted to verify our theory and demonstrate the superiority of the proposed method.
\end{abstract}

\section{Introduction}
Deep neural networks (DNNs) have been successfully and  widely used in many computer vision areas, such as pose estimation \cite{toshev2014,alp2018}, object detection \cite{redmon2016,hu2018} and super-resolution \cite{kim2016,lin2018}. Despite their excellent performance under the standard setting, recently, researchers found that DNNs are vulnerable to some well-designed pixel-wise perturbations. Those perturbations are invisible to human, whereas they are able to fool the network with high probability. For example, some attack methods, such as fast gradient sign method (FGSM) \cite{goodfellow2014} and project gradient descent (PGD) \cite{madry2017towards} can reduce the accuracy of the network to almost $0\%$ on CIFAR-10 dataset.  


To reduce the adversarial vulnerability, some adversarial defense methods are proposed. They can be roughly divided into three categories, $i.e.$, adversarial training based defense \cite{madry2017towards,zhang2019}, detection based defense \cite{metzen2017,liu2019a}, and reconstruction based defense \cite{samangouei2018}. 
Among these methods, one of the most straightforward ways is to analyze the loss surface with respect to the input, since the adversarial attack is to find the worst-case perturbation in the input space. Specifically, in \cite{pmlr-v97-simon-gabriel19a}, the relationship between the average value of the chosen norm of the gradient regarding to input space and the adversarial vulnerability was discussed, based on which the authors proposed a defense method by regularizing that value. 
Recently, Qin et al. \cite{qin2019adversarial} pointed out that the previous method did not consider the local characteristics, and it could have a relatively poor performance when the loss surface was not linear enough in the local area. 
Instead, they proposed to achieve the local linearity based on regularizing the difference between the loss and its first-order Taylor expansion.

Is the local linearity really necessary for the adversarial defense? In this paper, we make a systematic discussion on the relationship between adversarial robustness and the local flatness of loss surface. To be more precise, the local flatness can be measured by the maximum value of the chosen norm of the gradient with respect to the input within a neighborhood centered on the benign sample. We prove that whether a sample is easy to be attacked is related to the flatness of loss surface around that sample. Based on this discussion, we propose a novel gradient-based regularization, the local flatness regularization (LFR) to enhance the adversarial robustness. Besides, we compare our method with the human visual mechanism and the local Lipschitz property to further verify the validity of the method. And discussion on the relationship between our LFR and other previous related defense methods is theoretically conducted, which demonstrates that most of them are special cases of LFR under certain conditions.

The main contributions of this paper can be summarized as follows:

\begin{itemize}
    \item We give a systematic discussion on the relationship between the local flatness of loss surface and the adversarial robustness. Based on the analysis, we propose a new regularization, the LFR, for the adversarial defense.
    \item We theoretically discuss the relationship between LFR and previous related defense methods.
    \item Experiments and comparison with the human visual mechanism and the local Lipschitz property are included, which further verify the validity of the proposed method.
\end{itemize}

\section{Local Flatness Regularization}

\subsection{Preliminaries}
Suppose $L(\cdot)$ is the loss function (such as cross-entropy or the K-L divergence), and $B_p(\bm{x},\epsilon)$ is the $\epsilon$-ball centering around $\bm{x}$ under $\ell^p$ norm, $i.e.$, $B_p(\bm{x},\epsilon)= \{\bm{x}'|\  ||\bm{x}'-\bm{x}||_p \leq \epsilon \}$. 

In this paper, we focus on the defense of $\ell^\infty$ attack. This defense is representative, since $B_p(\bm{x},\epsilon) \subset B_\infty(\bm{x},\epsilon), (\forall p \geq 0)$, and therefore the adversarial robustness under $\ell^{\infty}$ norm indicates the adversarial robustness under any another norm.

\begin{defn}
The local flatness of the loss surface (generated by classifier $C$) around $B_{\infty}(\bm{x},\epsilon)$ is defined as 
\begin{equation}
    \gamma_C(\bm{x},\epsilon) = \max_{\bm{x}' \in B_{\infty}(\bm{x},\epsilon)} ||\partial_{\bm{x}}L(\bm{x}')||_1.
\end{equation}
\end{defn}


The reason why $\gamma_C(\bm{x},\epsilon)$ measures the local flatness can be easily explained. By the definition, gradient (at point $\bm{x}'$) indicates the direction in which the loss is changed at the highest rate. As such, $||\partial_{\bm{x}}L(\bm{x}')||_1$ can be regarded as the fluctuation at $\bm{x}'$, and therefore $\max_{\bm{x}' \in B_{\infty}(\bm{x},\epsilon)} ||\partial_{\bm{x}}L(\bm{x}')||_1$ measures the flatness of local $B_{\infty}(\bm{x},\epsilon)$. 

In the following part, we analyze the relationship between the local flatness and adversarial robustness theoretically

\begin{theorem}\label{1}
$\forall \bm{x}' \in B_\infty(\bm{x},\epsilon), L( \bm{x}')\leq L(\bm{x}) + \epsilon \cdot \gamma_C(\bm{x},\epsilon)$.
\end{theorem}

\begin{proof}
$\forall \bm{x}' \in B_\infty(\bm{x},\epsilon), \forall t \in [0,1]$, we have
$$
t\bm{x}'+(1-t)\bm{x} \in B_\infty(\bm{x},\epsilon).
$$
Therefore 

\begin{align*}
   L(\bm{x}^\prime) & = L(\bm{x}) + \int_{0}^{1} \frac{\partial   L(t\bm{x}^\prime +(1-t)\bm{x})}{\partial{t}} dt\\
             & = L(\bm{x}) + \int_{0}^{1} \frac{\partial L(t\bm{x}^\prime+(1-t)\bm{x})}{\partial{\bm{x}}} \cdot \frac{\partial{\bm{x}}}{\partial{t}} dt\\
            & = L(\bm{x}) + \int_{0}^{1} \nabla_{\bm{x}}L(t\bm{x}^\prime+(1-t)\bm{x}) \cdot (\bm{x}^\prime-\bm{x}) dt\\
          & \leq L(\bm{x}) + \epsilon \cdot \max_{\bm{x}^\prime \in B_\infty(\bm{x},\epsilon)} ||\partial_{\bm{x}}L(\bm{x}^\prime)||_{1}.
\end{align*}



\end{proof}

Theorem \ref{1} indicates that we can defend the attack under $\ell^\infty$ norm by regularizing the local flatness. 

Except for the previous perspective, its effectiveness can also be verified through the following aspects:

\newpage
\noindent \textbf{Verification from the aspect of human visual mechanism}
\vspace{0.3em}

It is widely accepted that the human visual system relies mainly on key components rather than all pixels of the whole image. For example, when categorizing a picture as a cat, the eye only pays attention to the pixels of the cat and ignores the background (such as the grass or the house). 

Considering the gradient of the loss function with respect  to  each  pixel  of  the  image $\bm{x}$, $i.e.$, the $\partial_{\bm{x}}L(\bm{x})$. The (absolute value of) gradient of the loss function at the pixel measures how important the pixel is to the prediction. Although there is no well-developed metric to identify which pixels are important to human vision system, using only key pixels at least means that the $\partial_{\bm{x}}L(\bm{x}) $ should be sparse, which can be constrained by the $\ell^1$ norm. This is also consistent with the phenomenon that the salience map is significantly more human-aligned for adversarially trained networks, as observed in \cite{nobug}. In addition, since the human visual system is not sensitive to small pixel-wise changes, $i.e.$ it is robust under certain pixel-wise perturbations, taking the local property of the gradients into consideration is rational.

\vspace{0.5em}
\noindent \textbf{Verification from the aspect of local Lipschitz property}

\begin{lemma}\label{Lip}
   Let $L(\cdot)$ be a Lipschitz continuous function, then $\max_{\bm{x}' \in B_\infty(\bm{x},\epsilon)} ||\partial_{\bm{x}}L(\bm{x}')|| \leq Lip(L)$, where $Lip(L)$ is the Lipschitz constant of loss $L$ in the local $B_\infty(\bm{x},\epsilon)$.  \rm{\cite{evans2015}}
\end{lemma}

Lemma \ref{Lip} indicates that regularizing $\gamma(\bm{x},\epsilon)$ has a direct connection with regularizing $Lip(L)$, which is effective for the defense since $\forall \bm{x}' \in B_\infty(\bm{x},\epsilon),$

$$
||L(\bm{x}') - L(\bm{x})|| \leq ||\bm{x}' - \bm{x}||\cdot Lip(L) \leq \epsilon \cdot Lip(L), 
$$
according to the definition of Lipschitz property.

\subsection{Proposed method}
Following the analysis above, we propose a novel method to defend $\ell^\infty$ attack as follows:
\begin{equation}\label{eq2}
    \underset{\boldsymbol{\theta}}{\min}\ \mathbb{E}_{(\bm{x},y) \in \mathcal{D}} \{L_{normal}(\bm{x},y) + \lambda \cdot \underbrace{\max_{\bm{x}' \in B_\infty(\bm{x},\epsilon)} ||\partial_{\bm{x}}L(\bm{x}')||_1}_{LFR} \},
\end{equation}
where $ L_{normal}(\cdot)$ is the normal training loss (such as cross-entropy or K-L divergence), and $\lambda$ is a non-negative hyperparamter.  

The minimax problem (\ref{eq2}) can be solved by alternatively solving the inner-maximization and the outer-minimization sub-problems as follows: 
\begin{itemize}
\item {\bf Inner-maximization}: given model parameters $\boldsymbol{\theta}$, for each $\bm{x} \in \mathcal{D}$, we generate an adversarial example $\bm{x}'$ by
\begin{flalign}\label{eq_max}
\bm{x'} \leftarrow \underset{\bm{x}' \in B_{\infty}(\bm{x},\epsilon)}{\arg\max}||\partial_{\bm{x}}L(\bm{x}';\boldsymbol{\theta})||_1. 
\end{flalign}
Equation (\ref{eq_max}) could be solved by different adversarial attack methods, such as PGD \cite{madry2017towards}. 
\item {\bf Outer-minimization}: given $\bm{x'}$, for each $\bm{x} \in \mathcal{D}$, the parameter $\boldsymbol{\theta}$ is updated by
\begin{flalign}
\boldsymbol{\theta} \leftarrow & ~
\underset{\boldsymbol{\theta}}{\arg\min}\ \mathbb{E}_{(\bm{x},y) \in \mathcal{D}} \{L(\bm{x},y;\boldsymbol{\theta}) +
\label{eq: outer-opt}
\\
& ~
 \lambda \cdot ||\partial_{\bm{x}}L(\bm{x}';\boldsymbol{\theta})||_1\}.
\nonumber
\end{flalign}
We update $\boldsymbol{\theta}$ using back-propagation \cite{rumelhart1985learning} with the stochastic gradient descent \cite{zhang2004solving}.
\end{itemize}



\section{Comparison with related methods}
In this section, we compare our method with some previous related defense methods under $\ell^\infty$ attack. Specifically, we prove that both adversarial training and first-order based adversarial defense are special cases of our method under certain conditions.

\subsection{Link to the first-order based defense}

\begin{defn}
To defend the attack under $\ell^\infty$ norm, first-order based adversarial defense is to minimize $L(\bm{x}) + \lambda ||\partial_{\bm{x}}L(\bm{x})||_1$. \rm{\cite{pmlr-v97-simon-gabriel19a}}
\end{defn}

\begin{theorem}\label{first}
First-order based adversarial defense is a special case of LFR with $\epsilon = 0$
\end{theorem}

\begin{proof}
$$
    ||\partial_{\bm{x}}L(\bm{x})||_{1} = 
    \max_{\bm{x}' \in B_{\infty}(\bm{x},0)} ||\partial_{\bm{x}}L(\bm{x}')||_1.
$$
\end{proof}

\subsection{Link to the adversarial training}

\begin{lemma}\label{lemma1}
   $\max_{||\bm{\alpha}||_p \leq 1} \bm{\alpha}\bm{\beta} = ||\bm{\beta}||_q, (\frac{1}{p} + \frac{1}{q} =1).$ \rm{ \cite{boyd2004convex}}
\end{lemma}

\begin{theorem}\label{adv}
Adversarial training using $\epsilon$-scaled FGSM attack under $\ell^\infty$ norm is equivalent to minimizing $L(\bm{x}) + \epsilon \cdot  \max_{\bm{x}' \in B_\infty (\bm{x},0)} ||\partial_{\bm{x}}L(\bm{x}')||_1$ up to terms of order $\epsilon^2$.
\end{theorem}

\begin{proof}
According to the first-order Taylor expansion, 
\begin{equation}
L(\bm{x}+ \bm{\alpha}) = L(\bm{x}) + \bm{\alpha} \cdot \partial_{\bm{x}}L(\bm{x}) + \mathcal{O}(\bm{\alpha}^2).
\end{equation}

Therefore the optimal adversarial training using $\epsilon$-scaled FGSM attack can be solved approximately up to terms of order $\epsilon^2$ by
\begin{equation}\label{eq4}
    \max_{||\bm{\alpha}||_{\infty} \leq \epsilon} L(\bm{x}) + \bm{\alpha}\cdot \partial_{\bm{x}}L(\bm{x}) = L(\bm{x}) + \max_{||\bm{\alpha}/{\epsilon}||_{\infty} \leq 1}\epsilon \cdot \frac{\bm{\alpha}}{\epsilon} \partial_{\bm{x}}L(\bm{x}).
\end{equation}

According to Lemma \ref{lemma1}, 
\begin{equation}
    (\ref{eq4}) = L(\bm{x}) + \epsilon \cdot \  ||\partial_{\bm{x}}L(\bm{x})||_1.
\end{equation}

In other words, adversarial training is a special case of first-order defense to a certain extent. By Theorem \ref{first}, the statement is proved.
\end{proof}

\begin{figure*}[!htb]
\vskip -0.13in
\centering
\subfigure[]{
\label{fig1a}
\includegraphics[width=0.49\textwidth]{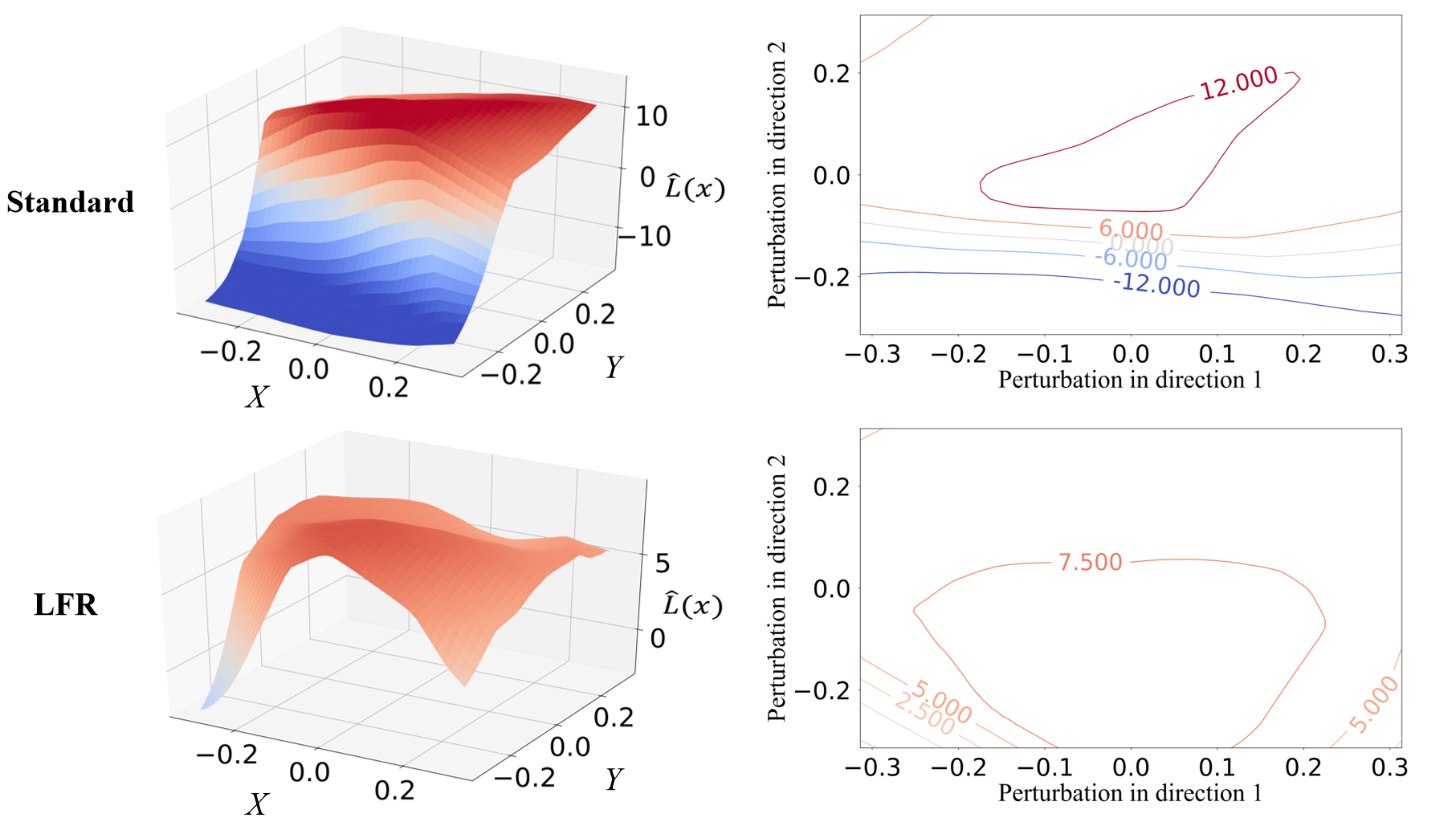}}
\subfigure[]{
\label{fig1b}
\includegraphics[width=0.49\textwidth]{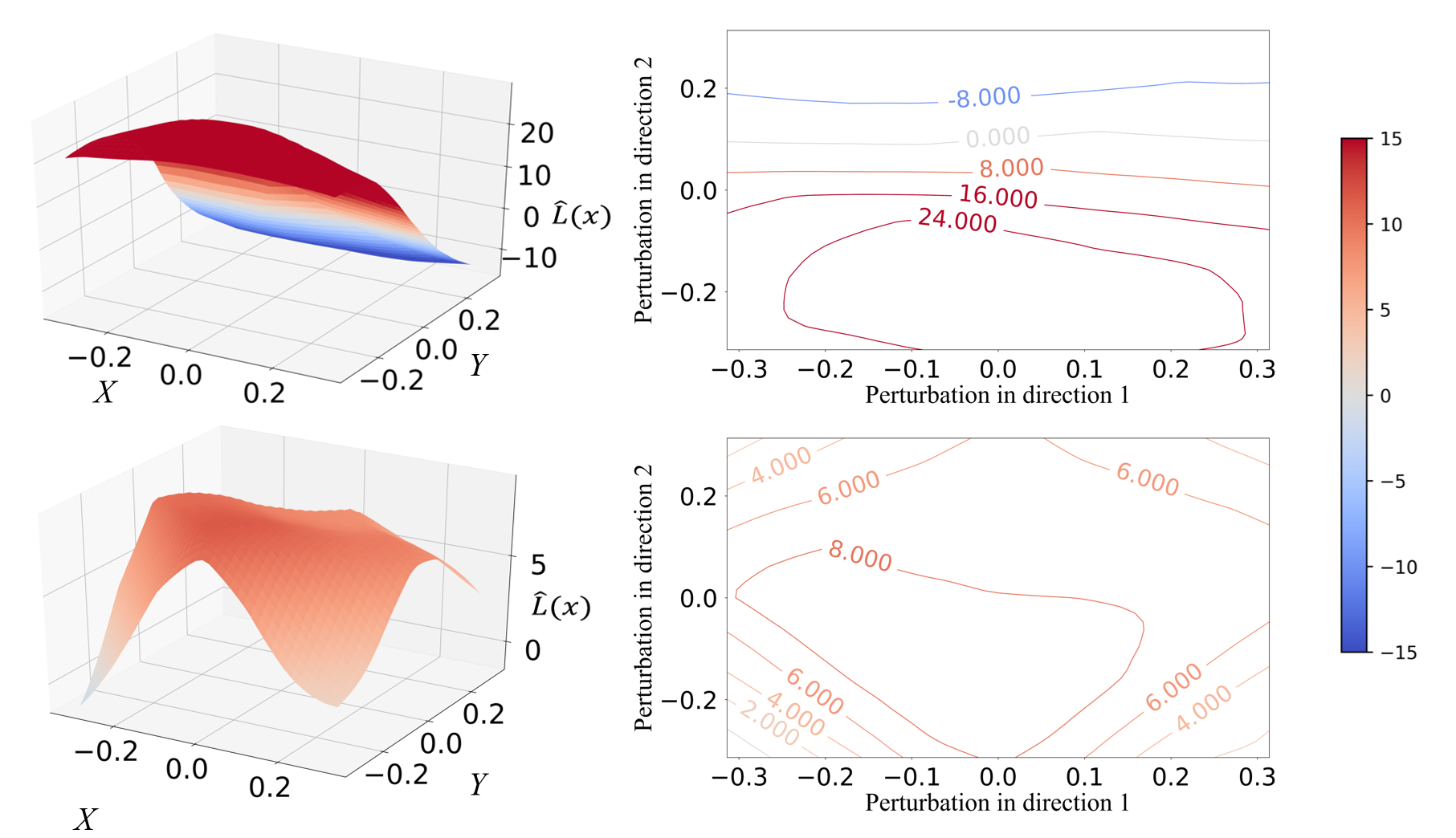}}
\vskip -0.13in
\caption{Comparison between the decision surfaces of the Standard and LFR models on two randomly selected samples in the MNIST dataset. First row: the 3D decision surfaces and their corresponding 2D version of the Standard model. Second row: the decision surfaces of the LFR model. In the 3D version, X and Y axis represent two different perturbation directions with the value indicating the perturbation size. Z-axis indicates the decision value. The prediction is correct if and only if the decision value is positive, which is represented as red areas. 
} 

\label{caption}
\vskip -0.1in
\end{figure*}

\section{Experiments}

\subsection{Adversarial defense under white-box attacks}

\noindent \textbf{Baselines Selection.}
We select trade-off inspired adversarial defense via surrogate-loss minimization (TRADES) \cite{zhang2019}, local linearization regularization (LLR) \cite{qin2019adversarial} and PGD-based adversarial training (AT) \cite{madry2017towards} as the baseline methods in the following experiments, since they are the representatives of the state-of-the-art defense methods, the most advanced loss surface geometry based defenses, and the most classical defense methods, respectively. Besides, we also train a model with the standard training process, dubbed ``Standard".

\noindent \textbf{Training Setup.} We conduct the experiments in the MNIST dataset \cite{mnist}, and adopt a simple CNN architecture for training. The simple CNN consists of four convolutional layers followed by three fully-connected layers. Specifically, we set the perturbation $\epsilon_{1}  = 0.3$, the perturbation step size $\eta_{1} = 0.01$, number of iterations $K_{1} = 40$ in inner maximization problem, and run 100 epochs with learning rate $\alpha_{1} = 0.01$ and batch size $m_{1}=128$. These settings are learned from \cite{zhang2019}. The hyperparamter $\lambda$ of the proposed LFR is set to 0.02. For the hyperparameters of other defenses, we set $1/\lambda=1$ for TRADES, $\lambda=4$ and $\mu=3$ for LLR according to the setting suggested in their papers.

\noindent \textbf{Attack Setup.} We evaluate the adversarial robustness of different methods under fast gradient sign method (FGSM) \cite{goodfellow2014}, project gradient descent (PGD) \cite{madry2017towards}, momentum iterative fast gradient sign method (MI-FGSM) \cite{MI-FGSM} and decoupled direction and norm attack (DDNA) \cite{DDN}. The attack setting we apply here is as follows: iteration $K_{t} = 40$, step size $\eta _{t} = 0.01$ and the maximum perturbation $\epsilon _{t} = 0.3$.

As shown in the Table \ref{tab:new_loss}, the accuracy of the Standard model drops drastically under these attacks whereas LFR still achieves valid performance. When compared with other baselines, LFR performs best under all attacks and obtains great improvements. Under stronger attacks $i.e.$ the $\text{PGD}^{40}$ and $\text{MI-FGSM}$, LFR outperforms the SOTA baselines by even larger margins.

\begin{table}[H]
\small
\vskip -0.13in
\centering
\caption{Robustness evaluation under various attacks.}
\scalebox{0.8}{
\begin{tabular}{c|c|cccc}
\toprule
\multirow{2}*{Defense}&\multirow{2}*{Clean}&\multicolumn{4}{c}{Attack Type}\\
\cline{3-6}
& & $\text{FGSM}$&$\text{PGD}^{40}$&$\text{MI-FGSM}$&$\text{DDNA}$\\
\midrule
Standard &$99.30\%$ &$33.69\%$ &$2.04\%$& $2.55\%$& $14.16\%$\\
AT &$99.52\%$ &$97.26\%$ &$95.37\%$&$94.17\%$&$94.05\%$\\
TRADES & $99.49\%$ & $97.54\%$ & $95.71\%$& $94.79\%$&$95.91\%$\\
LLR&$99.62\%$ &$97.94\%$& $95.63\%$&$94.60\%$&$93.95\%$\\
LFR& $99.47\%$&$\bf{98.14\%}$& $\bf{96.82\%}$& $\bf{96.01\%}$& $\bf{96.89\%}$\\
\bottomrule
\end{tabular}
}
\label{tab:new_loss}
\vskip -0.13in
\end{table}

\subsection{Visualization of decision surfaces}

\begin{figure}[htp]
    \centering
    \includegraphics[width=8cm]{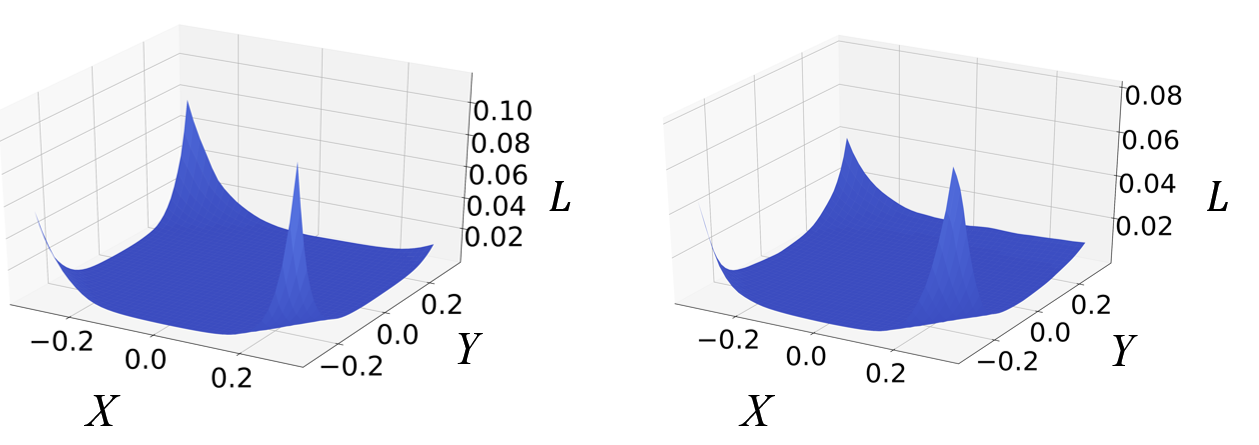}
    \caption{Loss surfaces of the LFR model. X and Y represent perturbation size along two directions. Z-axis indicates the value of loss function. }
    \label{loss_1}
\end{figure}

In this section, we analyze the effectiveness of LFR from the geometrical property of the decision surface. In particular, we visualize the decision surfaces of chosen models by using the method proposed in \cite{Yu2019Interpreting}. The decision value $\hat L(\bm{x})$ of a sample $\bm{x}$ is defined as $\hat L(\bm{x}) = p_{y} - \max_{i\neq{y}}p_{i}$, where $p_{i}$ is the logit value of label $i$, and $y$ is the ground-truth label. Hence, this value can evaluate the decision confidence with $\hat L(\bm{x})>0$ indicating that the prediction of $\bm{x}$ is correct and vice versa. We visualize the decision surfaces of the proposed LFR ($\lambda=0.02$) and another model with standard training, i.e. the Standard model ($\lambda=0$) on two randomly selected samples in MNIST dataset, as shown in Fig.\ref{caption}.

As shown in Fig.\ref{caption}, the decision surfaces between these two models behave quite differently. Compared with LFR, the decision surfaces of the Standard model have sharper peaks and larger slopes, implying that the decision is vulnerable to small perturbations. In other words, the decision confidence can quickly drop to negative areas when the model is fooled after being attacked by small pixel-wise adversarial perturbations. In contrast, the decision surfaces of LFR are rather flat and locate on a plateau with positive decision confidence in the vicinity of the sample. As such, the outputs of LFR still lie in the correct classification regions after being attacked.

We also visualize the loss surfaces $L(\bm{x})$ of the LFR model. Results on two randomly selected samples are shown in Fig. \ref{loss_1}. The loss surfaces are flat rather than linear, while the corresponding model is still robust. Therefore, it is the local flatness rather than the local linearity that is critical to adversarial defense.

\subsection{The effect of hyperparameter}
In this section, we further analyze how $\lambda$ could affect the performance.

\begin{figure}[htp]
    \centering
    \vskip -0.1in
    \includegraphics[width=7cm]{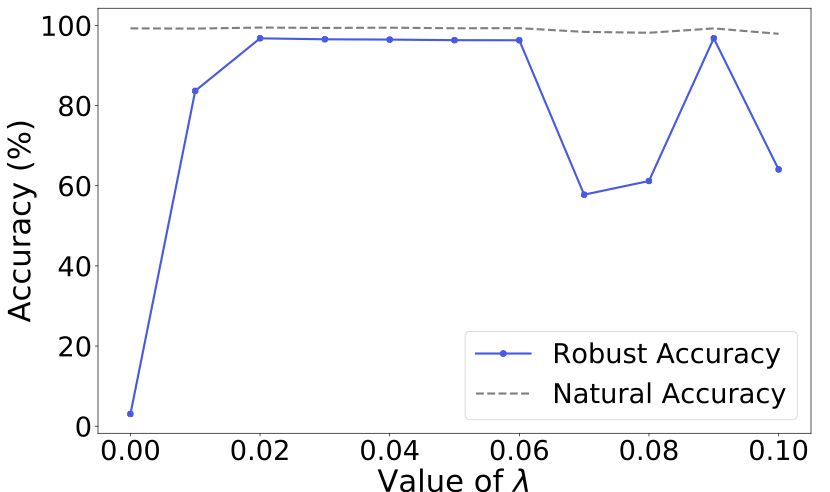}
    \caption{Clean and adversarial accuracy of the LFR model w.r.t different $\lambda$. }
    \label{fig:my_label}
    \vskip -0.1in
\end{figure}
\label{sec:pagestyle}

As shown in the Fig. \ref{fig:my_label}, the improvement of adversarial robustness led by LFR is significant, especially when the $\lambda$ is well-selected.

\section{Conclusions}
In this paper, we propose a new gradient-based regularization, the local flatness regularization (LFR), based on the relationship between the adversarial vulnerability and the local flatness of loss surface. The local flatness is defined as the maximum value of the chosen norm of the gradient regarding to the input within a neighborhood centered on the benign sample in the paper. We theoretically discuss the relationship between LFR with previous related defense methods, and further verify the effectiveness from both the aspect of human visual mechanism and local Lipschitz property. Verification experiments are conducted, which demonstrates the superiority of the proposed method.

\bibliographystyle{IEEEbib}
\bibliography{strings,refs}

\end{document}